\documentclass{llncs}
\usepackage{xcolor}

\usepackage{caption}
\usepackage{mdframed}
\usepackage{graphicx}
\usepackage{xfrac}
\usepackage{booktabs}
\usepackage{color, colortbl}
\usepackage{enumitem}
\usepackage{color, colortbl}

\pagecolor{white}

\usepackage{graphicx, subfigure, url}
\usepackage{wrapfig}
\usepackage{amsmath,amssymb}
\usepackage{epstopdf}
\newtheorem{condition}{Condition}

\def\R{\mathcal{R}} 
\def\V{\mathcal{V}}

\begin{document}
\title{Penalised FTRL With Time-Varying Constraints}
\author{Douglas J. Leith$^1$, George Iosifidis$^2$\\$^1$Trinity College Dublin,  Ireland $^2$TU Delft, The Netherlands}
\institute{}
\date{}                                           % Activate to display a given date or no date
\maketitle

\begin{abstract}
In this paper we extend the classical Follow-The-Regularized-Leader (FTRL) algorithm to encompass time-varying constraints, through adaptive penalization. We establish sufficient conditions for the proposed Penalized FTRL algorithm to achieve $O(\sqrt{t})$ regret and violation with respect to strong benchmark $\hat{X}^{max}_t$.  Lacking prior knowledge of the constraints, this is probably the largest benchmark set that we can reasonably hope for.  Our sufficient conditions are necessary in the sense that when they are violated there exist examples where $O(\sqrt{t})$ regret and violation is not achieved.   Compared to the best existing primal-dual algorithms, Penalized FTRL substantially extends the class of problems for which $O(\sqrt{t})$ regret and violation performance is achievable. 
\end{abstract}
\keywords{FTRL \and online convex optimisation \and constrained optimisation}   

\section{Introduction}

The introduction of online convex optimization (OCO) \cite{zinkevich} offered an effective way to tackle online learning and dynamic decision problems, with applications that range from portfolio selection, to routing optimization and ad placement, see \cite{hazan-book}. One of the seminal OCO algorithms is the Follow-The-Regularized-Leader (FTRL), which includes online gradient descent and mixture of experts as special cases. Indeed, FTRL is widely used today and has been studied in different contexts, e.g., with linear or non-linear objective functions, composite objectives, budget constraints, etc., see \cite{mcmahan-survey17}.

The general form of the FTRL update is:
\begin{align}
x_{\tau+1} \in \arg\min_{x\in X} \left\{ R_\tau(x)+\sum_{i=1}^\tau F_i(x)\right\} \label{eq:ftrl_update}
\end{align}
where action set $X\subset \mathbb{R}^n$ is bounded, function $F_i:X\rightarrow\mathbb{R}$ and regularizer $R_\tau:X\rightarrow\mathbb{R}$ is strongly convex. When the sum-loss $\sum_{i=1}^\tau F_i(x)$ is convex and $F_i(x)$ and $\big(R_i(x)-R_{i-1}(x)\big)$ are uniformly Lipschitz, the FTRL-generated sequence $\{x_\tau\}_{\tau=1}^t$ induces regret $\sum_{i=1}^t \big(F_i(x_i)-F_i(x)\big)\le O(\sqrt{t})$,  $\forall x \!\in \! X$, cf. \cite{mcmahan-survey17}.   Importantly, the set $X$ of admissible actions must be fixed and this is intrinsic to the method of proof, i.e., it is not a minor or incidental assumption.

The focus of this paper is to extend the FTRL algorithm in order to accommodate time-varying action sets, i.e., cases where at each time $\tau$ the fixed set action $X$ is replaced by set $X_\tau$ which may vary over time.  We refer to this extension to FTRL as \emph{Penalised FTRL}.

In general, it is too much to expect to be able to simultaneously achieve $O(\sqrt{t})$ regret and strict feasibility $x_{\tau}\in X_{\tau}$, $\tau=1,\dots,t$. We therefore allow limited violation of the action sets $\{X_{\tau}\}$ and instead aim to simultaneously achieve $O(\sqrt{t})$ regret and $O(\sqrt{t})$ constraint violation.  That is, defining loss function $f_\tau:D\rightarrow\mathbb{R}$ on domain $D\subset \mathbb{R}^n$ and constraint functions $g_{\tau}^{(j)}:D\rightarrow\mathbb{R}$ such that $X_\tau=\big\{x\in D: g_{\tau}^{(j)}(x)\le 0, j=1,\dots,m \big\}$ then we aim to simultaneously achieve regret and violation:
\[
\R_t=\sum_{i=1}^t \Big(f_i(x_i)-f_i(x)\Big)\le O(\sqrt{t}), \  \V_t=\sum_{j=1}^m\max\{0,\sum_{i=1}^t g_{i}^{(j)}(x_i)\}\le O(\sqrt{t})
\]
for all $x\in X^{max}_t:=\big\{x\in D: \sum_{i=1}^t g_{i}^{(j)}(x)\le 0, j=1,\dots,m \big\}$.

\emph{Importance of Using A Strong Benchmark.} We know from \cite{JMLR:v10:mannor09a} that $O(\sqrt{t})$ regret and violation with respect to benchmark set $X^{max}_t$ is not achievable for all possible sequences of constraints $\{g_{i}^{(j)}\}$.   It is therefore necessary to: \emph{(i)} change the benchmark set $X^{max}_t$ to something more restrictive; or \emph{(ii)} restrict the admissible set of constraint sequences $\{g_{i}^{(j)}\}$; or \emph{(iii)} both.   In the literature, it is common to adopt the weaker benchmark:
$$
X^{min}_t:=\left\{x\in D: g_{i}^{(j)}(x)\le 0, i=1,\dots,t, j=1,\dots,m\right\}\subset X^{max}_t
$$
i.e., to focus on actions $x$ which \emph{simultaneously satisfy every constraint at every time}. But this weak benchmark is in fact so restrictive and easy for a learning algorithm to outperform, where the achieved regret $\R_t$ is often negative in practice, and indeed $-\R_t\le O(t)$.   

One of our primary interests, therefore, is in retaining a benchmark that is close to $X^{max}_t$.   To this end, we consider the following benchmark:
$$
\hat{X}^{max}_t:=\bigg\{x\in D: \sum_{i=1}^\tau g_{i}^{(j)}(x)\le 0, \forall j\leq m, \tau\leq t\bigg\}.
$$ 

\begin{wrapfigure}[11]{r}{0.33\columnwidth}
\centering
\vspace{-0.75cm}
\includegraphics[width=0.25\columnwidth]{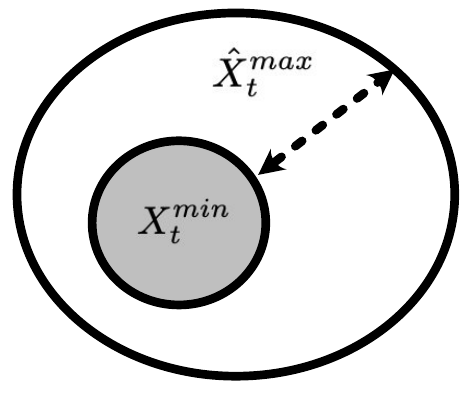}
\caption{Showing how our benchmark set $\hat{X}^{max}_t$ compares to $X^{min}_t$.}\label{fig:sets}
\end{wrapfigure}

\noindent We can see immediately that $X^{min}_t\subset \hat{X}^{max}_t$.  The set $\hat{X}^{max}_t$ requires $\sum_{i=1}^\tau g_{i}^{(j)}(x)\le 0$ to hold at every time $\tau\leq t$ rather than just at the end of the horizon $t$, and so is still smaller than $X^{max}_t$. Lacking, however, predictions or prior knowledge of the constraints $g_{i}^{(j)}$, it is probably the best we can reasonably hope for.  To illustrate the difference between $\hat{X}^{max}_t$ and $X^{min}_t$, suppose the time-varying constraint is $x\le 1/\sqrt{t}$. Then $X^{min}_t=[0, 1/\sqrt{t}]$ which tends to set $\{0\}$ for $t$ large, $X^{max}_t=D=[0,1]$ for $t\ge 1$, and $\hat{X}^{max}_t=D=[0,1]$.

%%%%%%%%%%%%
\section{Related Work} \label{sec:related}

The literature on online learning with time-varying constraints focuses on the use of primal-dual algorithms (see update (\ref{eq:pm1}) in the sequel), and largely fails to obtain $O(\sqrt{t})$ regret and violation simultaneously even w.r.t. the weak $X^{min}_t$ benchmark. The standard problem setup consists of a sequence of convex cost functions $f_t:D\rightarrow\mathbb{R}$ and constraints $g_{t}^{(j)}:X\rightarrow\mathbb{R}$, $j=1,\dots,m$, where actions $x \in D\subset\mathbb{R}^n$.  The canonical algorithm performs a primal-dual gradient descent iteration, namely:
%\small
\begin{align}
\!\!\!x_{t+1} \!=\! \Pi_D\left(x_t - \eta_t (\partial f_t(x_t) \!+ \lambda_t^T\partial g_t(x_t))\right),\ %
\lambda_{t+1} \!=\! \Big[(1\!-\theta_t)\lambda_t \!+\! \mu_t g_t(x_{t+1})\Big]^+ \label{eq:algo1}
\end{align}

with step-size parameters $\eta_t$, $\mu_t $ and regularisation parameter $\theta_t$; while $\Pi_D(
\alpha)$ denotes the project of $\alpha$ onto $D$.  Commonly, the parameter $\theta_t \equiv 0$, with exceptions being \cite{JMLR:v13:mahdavi12a}, \cite{pmlr-v48-jenatton16},  and \cite{pmlr-v70-sun17a} that employ non-zero $\theta_t$. \cite{yu2017} approximate $ g_t(x_{t+1})$ in the $\lambda_{t+1}$ update by  $g_t(x_{t}) +\partial g_t(x_t)(x_{t+1}-x_t)$.

The $\R_t$ is commonly measured w.r.t. the baseline action set $X^{min}_t=\{x \in D: g_{i}^{(j)}(x) \le 0,i=1,\dots,t, j=1,\dots,m\}$, with the exception of \cite{pmlr-v108-valls20a} where a slightly larger set is considered; \cite{yu2017} that considers stochastic constraints and the baseline action set is $\{x\in D: E[g_{i}^{(j)}(x)]\le 0, j=1,\dots,m\}$; and \cite{pmlr-v97-liakopoulos19a} which considers a $K$-slot moving window for the sum-constraint satisfaction.

The original work on this topic restricted attention to time-invariant constraints $g_{i}^{(j)}(x)\!=\!g^{(j)}(x)$. With this restriction, the work in \cite{pmlr-v48-jenatton16} achieves $\R_t\!\leq\! O(\max\{t^\beta,t^{1-\beta}\})$ and $\V_t\!\leq\! O(t^{1-\beta/2})$ constraint violation, which yields $\R_t$, $\V_t \!\leq \!O(t^{2/3})$ with $\beta\!=\!2/3$.  Similar bounds are derived in \cite{JMLR:v13:mahdavi12a}.  It is worth noting that these results are primarily of interest for their analysis of the primal-dual algorithm rather than the performance bounds per se, since classical algorithms such as FTRL are already known to achieve $O(\sqrt t)$ regret and no constraint violation for constant constraints. 
For general time-varying cost and constraint functions, \cite{pmlr-v70-sun17a} achieve $O(\sqrt t)$ regret and $O(t^{3/4})$ constraint violation; \cite{pmlr-v97-liakopoulos19a} achieve $\R_t= O(t^{1/2}+KT/V)$ and $\V_t=O((Vt)^{1/2})$, with $K=1$ corresponding to baseline set $X^{min}_t$ and $V$ a design parameter.   Selecting $V=t^{1/2}$ gives $O(t^{1/2})$ regret and $O(t^{3/4})$ constraint violation, similarly to \cite{pmlr-v70-sun17a}. By restricting the constraints, \cite{pmlr-v108-valls20a} improves this to $O(t^{1/2})$ regret and constraint violation.  As already noted, this requires restricting the constraints to be $g_{i}^{(j)}(x)=g^{(j)}(x)-b_{i,j}$ with $b_{i,j}\in\mathbb{R}$ i.e. the constraints are $g^{(j)}(x) \le b_{i}^{(j)}$ with time-variation confined to threshold $b_{i}^{(j)}$.  Yu et al~\cite{yu2017} also achieve $O(t^{1/2})$ regret and  \textit{expected} constraint violation (i.e. $E[\sum_{i=1}^tg_{i}^{(j)}(x_t)]\le O(t^{1/2})$), this time by restricting the constraints to be i.i.d. stochastic.  Yi et al~\cite{Johansson2020} obtain  $O(t^{2/3})$ regret and constraint violation by restricting the cost and constraint functions to be separable.  Chen et al~\cite{chen2017} focus on a form of dynamic regret that upper bounds the static regret and show $o(t)$ regret and $O(t^{2/3})$ constraint violation under a slow variation condition on the constraints and dynamic baseline action.

%%%%%%%%%%%%
\section{Preliminaries}
%%%%%%%%%%%%
\subsection{Exact Penalties}
We begin by recalling a classical result of Zangwill~\cite{zangwill67}.  Consider the convex optimisation problem $P$: 
$$
\min_{x\in D} f(x) \qquad \text{s.t.}\qquad  g^{(j)}(x) \le 0,\ j=1,\cdots,m
$$
where $D\subset \mathbb{R}^n$, $f:\mathbb{R}^n \rightarrow \mathbb{R}$ and $g^{(j)}:\mathbb{R}^n \rightarrow \mathbb{R}$, $j=1,\cdots,m$ are convex.  Let $X:=\{x: x\in D, g^{(j)}(x)\le0, j=1,\cdots,m\}$ denote the feasible set and $X^*\subset X$ the set of optimal points.  Define:
\begin{align}
F(x) := f(x) +  \gamma\sum_{j=1}^m\max\left\{0,g^{(j)}(x)\right\},\quad \gamma \in \mathbb R. \label{eq:pen}
\end{align}
$F(x)$ is convex since $f(\cdot)$, $g^{(j)}(\cdot)$ are convex and $\max\{\cdot\}$ preserves convexity.  

\begin{figure}[tb]
\centering
\includegraphics[width=0.4\columnwidth]{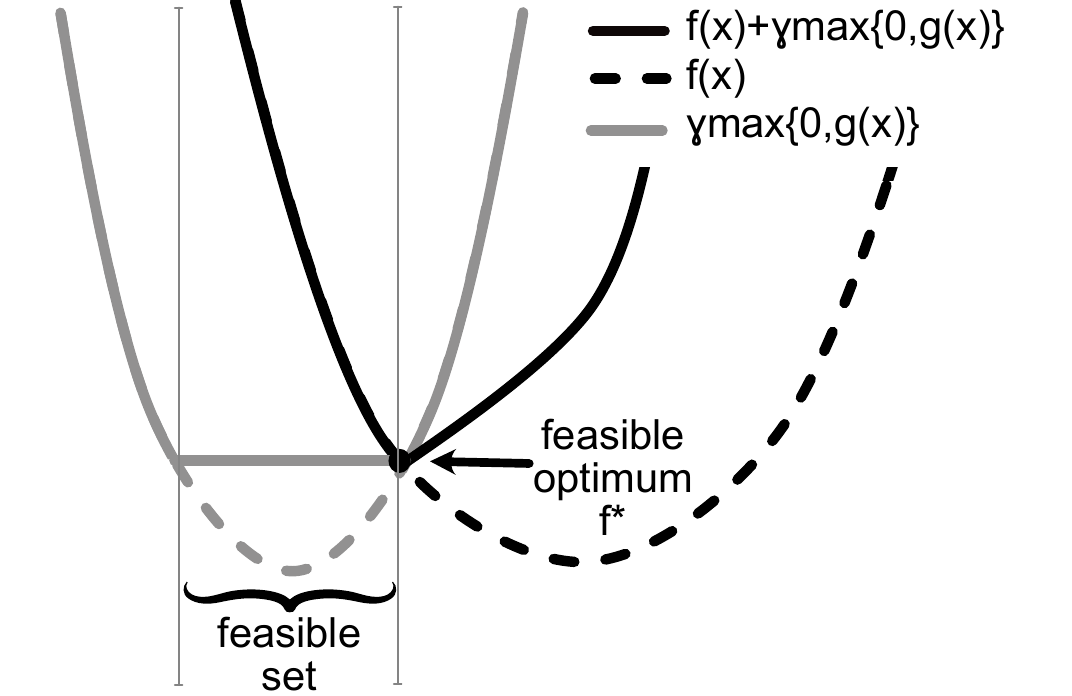}
\caption{Illustrating use of a penalty to convert constrained optimization $\min_{x:g(x)\!\le\! 0} f(x)$ into unconstrained optimization $\min_x f(x)\!+\gamma\max\{0,g(x)\}$, $\gamma\!>\!0$.  Within the feasible set $g(x)\le0$ and $\gamma\max\{0,g(x)\}\!=\!0$.  Outwith this set $\gamma\max\{0,g(x)\}=\gamma g(x)\!>\!0$.  The idea is that $\gamma$ is selected large enough that outwith the feasible set $f(x)\!+\!\gamma\max\{0,g(x)\}> f^*$, the min value of $f$ inside the feasible set.}\label{fig:penalty_sketch}
\end{figure}

The key idea is that the penalty (second term in \eqref{eq:pen}) is zero for $x\in X$, but large when $x\notin X$.   Provided $\gamma$ is selected large enough, the penalty  forces the minimum of $F(x)$ to \emph{(i)} lie in $X$ and \emph{(ii)} match $\min_{x\in X} f(x)$; see example in Fig. \ref{fig:penalty_sketch}. The next lemma, proved in the Appendix, corresponds to~\cite[Lemma 2]{zangwill67}.   
%%%%
\begin{lemma}[Exact Penalty]\label{lem:zangwill}
Assume that a Slater point exists i.e. a feasible point $z\in D$ such that $g^{(j)}(z)<0$, $j=1,\cdots,m$.  Let $f^*:=\inf_{x\in X} f(x)$ (the solution to optimization $P$) .  Then there exists a finite threshold $\gamma_0\ge 0$ such that $F(x)\ge  f^*$ for all $x\in D$, $\gamma \ge \gamma_0$, with equality only when $x\in X^*$.  It is sufficient to choose ${\gamma}_0 = \frac{f^*-f(z)-1}{\max_{j\in\{1,\cdots,m\}}\{g^{(j)}(z)\}}$.
\end{lemma}

%%%%%%%%%%%%
\subsection{FTRL Results}
We also recall the following standard FTRL results (for proofs see, e.g., ~\cite{shwartz-book}).

\begin{lemma}[Be-The-Leader]\label{lem:btl}
Let $F_i, i=1,\dots,t$ be a sequence of (possibly non-convex) functions $F_i:D\rightarrow \mathbb{R}$, $D\subset\mathbb{R}^n$.  Assume that $\arg\min_{x\in D}\sum_{i=1}^\tau F_i(x)$ is not empty for $\tau=1,\dots,t$.  Selecting sequence $w_{i+1}, i=1,\dots,t$ according to the Follow The Leader (FTL) update $w_{\tau+1} \in \arg\min_{x\in D} \sum_{i=1}^\tau F_i(x)$, ensures $\sum_{i=1}^t F_i(w_{i+1}) \le \sum_{i=1}^t F_i(y)$ for every $y\in D$.
\end{lemma}

\begin{condition}[FTRL]\label{cond:ftrl}
(i) Domain $D$ is bounded (potentially non-convex), (ii) $\sum_{i=1}^\tau F_i(x)$ is convex (the individual $F_i$'s need not be convex), (iii) $F_i(x)$ is uniformly $L_f$-Lipschitz on $D$ i.e. $|F_i(x)-F_i(y)|\le L_f\|x-y\|$ for all $x,y\in D$ and where $L_f$ does not depend on $i$, and (iv) $R_\tau(x)$ is $\sqrt{\tau}$-strongly convex and $\big(R_i(x)-R_{i-1}(x)\big)$ is uniformly Lipschitz, e.g. $\sqrt{\tau}\|x\|_2^2$.
\end{condition}
%
%%%%%%%%%%%%%%%
\begin{lemma}[Regret of FTRL]\label{th:ftrl}
When Condition \ref{cond:ftrl} holds, the sequence $\{x_\tau\}_{\tau=1}^t$ generated by the FTRL update
$
x_{\tau+1} \in \arg\min_{x\in D} R_\tau(x)+\sum_{i=1}^\tau F_i(x) 
$
has regret $\R_t=\sum_{i=1}^t F_i(x_i)-F_i(x)\le O(\sqrt{t})$ for all $x\in D$.
\end{lemma}

\begin{lemma}[$\sigma_\tau$-Strongly Convex Regulariser]\label{lem:reg}
When $\sum_{i=1}^\tau F_i(x)$ is $\sigma_\tau$-strongly convex, $F_i(x)$ uniformly $L_f$-Lipschitz over $D$ and $w_{\tau+1}\in \arg \min_{x\in D} \sum_{i=1}^\tau F_i(x)$, it holds $\|w_{\tau+1}-w_{\tau}\| \le 2L_f/(\sigma_\tau+\sigma_{\tau-1})$
\end{lemma}

%%%%%%%%%%%%
\section{Penalised FTRL}

%%%%%%%%%%%%
\subsection{Exact Penalties For Time-Invariant Constraints}
We begin by demonstrating the application of Lemma \ref{lem:zangwill} to FTRL update (\ref{eq:ftrl_update}) with time-invariant action set $X$.  Selecting
$
F_i(x)= f_i(x)+\gamma h(x)
$
with $h(x)= \sum_{j=1}^m\max\{0,g^{(j)}(x)\}$ and defining the bounded domain $D$ with $X\subset D$, then by standard analysis, cf. \cite{mcmahan-survey17}, the penalized FTRL update\footnote{Note the subtle yet crucial difference w.r.t. non-penalized FTRL update \eqref{eq:ftrl_update}.}:
\begin{align}
x_{\tau+1} \in \arg\min_{x\in D} \left\{ R_\tau(x)+\sum_{i=1}^\tau F_i(x)\right\} \label{eq:ftrl_update_pen}
\end{align}
ensures regret $\sum_{i=1}^t (F_i(x_i)-F_i(x))\le O(\sqrt{t})$ for all $x\in D$, and since $X\subset D$ for all $x\in X$.  Of course this says nothing about whether the actions $x_{i}$ lie in set $X$ nor anything much about the regret of $f_i(x_i)$, but when set $X$ has a Slater point and $\gamma$ is selected large enough then by Lemma \ref{lem:zangwill} we have that $x_{\tau+1}\in X$ for all $\tau$.   It follows that  $F_i(x_i)= f_i(x_i)$ (since $h(x_i)=0$ when $x_i\in X$) and so regret $\sum_{i=1}^t (F_i(x_i)-F_i(x))=\sum_{i=1}^t (f_i(x_i)-f_i(x))\le  O(\sqrt{t})$ for all $x\in X$.

%%%%%%%%%%%%
\subsection{Penalties For Time-Varying Constraints}
We now extend consideration to FTRL with time-varying constraints.  Our aim is to define a penalty which is zero on a set $\hat{X}_\tau^{max}\approx X_{max}$, and large enough outside this set to force the minimum of $\sum_{i=1}^\tau F_i(x)$ to lie in $\hat{X}^{max}$.

%%%%%%%%%%%%
\subsubsection{Penalties Which Are Zero When $x\in \hat{X}^{max}_\tau$.}

Consider extending the penalty-based FTRL (\ref{eq:ftrl_update_pen}) to time-varying constraints. We might try selecting:
\[
F_i(x)= f_i(x)+\gamma h_i(x), \quad \text{with} \quad h_i(x)= \sum_{j=1}^m\max\left\{0,g_i^{(j)}(x)\right\},
\]
but we immediately run into the following difficulty.   We have that
$
\sum_{i=1}^\tau F_i(x) = \sum_{i=1}^\tau f_i(x) +  \gamma\sum_{i=1}^\tau \sum_{j=1}^m\max\{0,g_i^{(j)}(x)\}
$
and so to make the second term zero requires $g_i^{(j)}(x)\le 0$ for all $i\leq\tau$ and $j\leq m$, i.e. requires every constraint over all time to simultaneously be satisfied.  This penalty choice  $h_i(\cdot)$ therefore corresponds to benchmark $X^{min}_t$, whereas our interest is in set $X^{max}_t$.  It is perhaps worth noting that this corresponds to the penalty used in the primal-dual literature, so it is unsurprising that those results are confined to $X^{min}$.

With this in mind, consider instead selecting 
\[
h_\tau(x)= \sum_{j=1}^m\max\left\{0,\sum_{i=1}^\tau g_i^{(j)}(x)\right\} -\sum_{j=1}^m\max\left\{0,\sum_{i=1}^{\tau-1} g_i^{(j)}(x)\right\}
\]
with $h_1(x)= \sum_{j=1}^m\max\{0,g_i^{(j)}(x)\}$.  Then,
\[
\sum_{i=1}^\tau F_i(x) = \sum_{i=1}^\tau f_i(x) +  \gamma\sum_{j=1}^m\max\left\{0, \sum_{i=1}^\tau g_i^{(j)}(x)\right\}.
\] 
We now have a sum-constraint in the second term, as desired.   Unfortunately, this choice of $h_i(\cdot)$ violates the conditions needed for FTRL to achieve $O(\sqrt{t})$ regret.  Namely, it is required that $F_i(\cdot)$ is uniformly Lipschitz but $h_i(\cdot)$ does not satisfy this condition, and so neither does $F_i(\cdot)$.   To see this, observe that when $g_i^{(j)}(\cdot)$ is uniformly Lipschitz with constant $L_g$, then $\sum_{i=1}^\tau g_i^{(j)}(x)$ has a Lipschitz constant $\tau L_g$ that scales with $\tau$, and so there exists no uniform upper bound.  The $\max$ operator in $h_i(\cdot)$ does not change the Lipschitz constant (see Lemma \ref{lem:maxlips}); thus $h_i(\cdot)$ is $\tau L_g$ Lipschitz, which prevents FTRL achieving $\R_t\leq O(\sqrt{t})$.

These considerations lead us to the following penalty,
\begin{align}
h_\tau(x)= \sum_{j=1}^m\max\left\{0,\frac{1}{\tau}\sum_{i=1}^\tau g_i^{(j)}(x)\right\}.  \label{eq:ftrl_pen}
\end{align}
When $g_i^{(j)}(\cdot)$ is uniformly Lipschitz with constant $L_g$ then so is $h_i(\cdot)$ due to the $1/\tau$ prefactor added to the sum and the following Lemma which just states that when a function $h(x)$ is $L$-Lipschitz then $\max\{0,h(x)\}$ is also $L$-Lipschitz:
\begin{lemma}\label{lem:maxlips}
When $|h(x)-h(y)|\le L\|x-y\|$ then $|\max\{0,h(x)\}-\max\{0,h(y)\}|\le L\|x-y\|$.  
\end{lemma}
\begin{proof}
Observe that $2\max\{0,h(x)\} = h(x) + |h(x)|$.  Therefore,
$
2|\max\{0,h(x)\} - \max\{0,h(y)\}| 
= |h(x) -h(y) + |h(x)|-|h(y)||
\le |h(x) -h(y)| + | |h(x)|-|h(y)| |
\le |h(x) -h(y)| + | h(x)-h(y)|  
\le 2L\|x-y\|
$.
\end{proof}
With this choice, we can write:
\begin{align*}
\sum_{i=1}^\tau F_i(x) &= \sum_{i=1}^\tau f_i(x) + \gamma \sum_{j=1}^m \sum_{i=1}^\tau\max\left\{0, \frac{1}{i}\sum_{k=1}^i g_k^{(j)}(x)\right\}.%\\
\end{align*}

The second term is zero when
\begin{align*}
x_\tau &\in \hat{X}^{max}_\tau:=\Big\{x\in D: \sum_{i=1}^\tau h_i(x)\le 0\Big\}
=\Big\{x:\sum_{k=1}^i g_k^{(j)}(x)\le 0, j\leq m, i\leq \tau\Big\}
\end{align*}

%%%%%%%%%%%%
\subsubsection{Penalties Which Are Large When $x\notin \hat{X}^{max}_\tau$}
In addition to requiring the penalty for time-varying constraints to be zero for $x \in \hat{X}^{max}_\tau$ we also require the penalty to large enough when $x \notin \hat{X}^{max}_\tau$ so as to force the minimum of $\sum_{i=1}^\tau F_i(x)$ to lie in set $\hat{X}^{max}_\tau$, or at least to only result in $O(\sqrt{\tau})$ violation.  

As already noted, to use FTRL we need $F_i(\cdot)$ to be uniformly Lipschitz, which requires $f_i(\cdot)$ to be uniformly Lipschitz.   When $f_i(\cdot)$ is $L_f$-Lipschitz then $|\sum_{i=1}^\tau f_i(x)|$ may grow linearly with $\tau$ at rate $\tau L_f$.   We therefore require the penalty $\sum_{i=1}^\tau h_i(x)$ to also grow at least linearly with $\tau$ since otherwise for all $\tau$ large enough $|\sum_{i=1}^\tau f_i(x)| \gg \sum_{i=1}^\tau h_i(x)$ and the penalty may become ineffective i.e. we can have $x_\tau \notin \hat{X}_\tau$ for all $\tau$ large enough and so end up with $O(t)$ constraint violation, which is no good. 

We formalize the requirement the sum-penalty $\sum_{i=1}^\tau h_i(x)$ in (\ref{eq:ftrl_pen}) needs to grow quickly enough as follows.   Let $\partial \hat{X}^{max}_\tau$ denote the boundary of  $\hat{X}^{max}_\tau$.   Let:
\[k_{\tau}:=\min_{x\in \partial \hat{X}^{max}_\tau}|\{(i,j): \frac{1}{i}\sum_{k=1}^i g_k^{(j)}(x)\ge 0, i\in\{1,\dots,\tau\}, j=\{1,\dots,m\}\}|.
\]
That is, $k_{\tau}$ is the minimum number of constraints active at the boundary of $\hat{X}^{max}_\tau$.   Observe that $1\le k_\tau \le \tau$ with, for example, $k_\tau=\tau$ when $g_{i}^{(j)}(x)=g^{(j)}(x)$ does not depend on $i$.

%%%%%%%%%%%%%%%
\begin{condition}[Penalty Growth]\label{cond:penalty1}
Let  $z\in D$ be a common Slater point such that $\frac{1}{\tau}\sum_{i=1}^\tau g_i^{(j)}(z)<-\eta<0$ for $j=1,\dots,m$ and $\tau>t_\epsilon$ (the same $z$ must work for all $\tau$ and $j$).  We require that $k_{\tau}\ge \frac{\beta}{\eta} \tau$ for all $\tau>t_\epsilon$, where $\beta>0$ and the same $\beta$ must work for all $\tau=1,\dots,t$.
\end{condition}

%%%%%%%%%%%%%%%%%%%%%%%%%%%%%%
\subsubsection{Time-Varying Exact Penalties}
We are now in a position to extend the penalty approach to time-varying constraints.  We begin by applying Lemma \ref{lem:zangwill} to optimisation problem $P^\prime$:
$
\min_{x\in D} f(x)\ 
s.t.\  \frac{1}{i}\sum_{k=1}^i g_{k}^{(j)}(x)\le 0,\ i=1,\cdots,t,\ j=1,\cdots,m
$ 
where  $f(\cdot)$ and $g_{i}^{(j)}(\cdot)$, $i=1,\dots,t$, $j=1,\cdots,m$ are convex and $D\subset \mathbb{R}^n$ is convex and bounded.  Let $C^*=\arg\min_{x\in \hat{X}^{max}_t} f(x)$.   Define
\begin{align*}
H(x) := f(x)+ \gamma\sum_{i=1}^t\sum_{j=1}^m\max\left\{0,\frac{1}{i}\sum_{k=1}^ig_{k}^{(j)}(x)\right\}
\end{align*}
where $\gamma\in\mathbb{R}$.  Note that $H(\cdot)$ is convex since $f(\cdot)$, $g_{i}^{(j)}(\cdot)$ are convex and composition with $\max$ preserves convexity.  
%%%%
\begin{lemma}\label{lem:two}
Assume a Slater point exists, i.e. a $z\!\in\! D$ such that $\frac{1}{i}\sum_{k=1}^i g_k^{(j)}(z)\!<\!-\eta<0$, $i=1,\dots,t$, $j=1,\dots,m$.  Let $f^*:=\min_{x\in \hat{X}^{max}_t} f(x)$.  Then there exists a finite threshold $\gamma_0\ge 0$ such that $H(x)\ge  f^*$ for all $x\in D$, $\gamma \ge \gamma_0$, with equality only when $x\in \hat{X}^{max}_t$.   It is sufficient to choose $\gamma_0 \ge \frac{f^*-f(z)-1}{-k_t\eta}$. 
\end{lemma}
\begin{proof}
Setting the expression for $\gamma_0$ to one side for now, the result follows from applying Lemma \ref{lem:zangwill} to $P^\prime$.  Turning now to expression $\gamma_0 \!\ge\! \frac{f^*-f(z)-1}{k_t\eta}$, comparing this with the expression in Lemma \ref{lem:zangwill}, observe that the only change is in the denominator, which applying Lemma \ref{lem:zangwill} to $P^\prime$ is $\max_{i\leq t, j\leq m}\{\frac{1}{i}\sum_{k=1}^ig_{k}^{(j)}(z)\}\!=\!-\eta$.  Referring to (\ref{eq:zangwillG}) in the proof of Lemma \ref{lem:zangwill}, it is sufficient the denominator $G$ of $\gamma_0$ is such that $\frac{\sum_{(i,j)\in A} \frac{1}{i}\sum_{k=1}^ig_{k}^{(j)}(z)}{G}\ge 1$, where $A\subset\{1,\dots,t\}\times\{1,\dots,m\}$.   By assumption $g_{k}^{(j)}(z)\!\le\! -\eta$ and so $\sum_{j\in A} \frac{1}{i}\sum_{k=1}^ig_{k}^{(j)}(z)\! \le\! -|A|\eta$ with $|A|\!\ge\!1$.  Now $k_t\!\in[1,|A|]$, thus suffices to see setting $G=-k_t\eta$ also meets this requirement. 
\end{proof}

%%%%%%%%%%%%%%%
\begin{theorem}[Time-Varying Exact Penalty]\label{th:constraint}
The sequence $x_\tau$, $\tau=1,\dots,t$ generated by the FTRL update (\ref{eq:ftrl_update_pen}) with
$
F_i(x)= f_i(x)+\gamma h_i(x)
$
and $h_i(x):=\sum_{j=1}^m\max\{0,\frac{1}{i}\sum_{k=1}^i g_k^{(j)}(x)\}$ satisfies $x_{\tau+1}\in \hat{X}^{max}_\tau$ for $\tau> t_\epsilon$ when Condition \ref{cond:penalty1} holds and parameter $\gamma>\frac{E+L+1}{\beta}$ where
$
E\ge \max_{y\in D,i\in\{1,\dots,t\}}(R_i(y)-R_i(z))/i,\ 
L\ge \max_{y\in D,i\in\{1,\dots,t\}} f_i(y)-f_i(z)
$
with $z\in D$  a Slater point. 
\end{theorem}
\begin{proof}
The result follows by application of Lemma \ref{lem:two} at times $\tau> t_\epsilon$ with $h(x)=R_\tau(x)+\sum_{i=1}^\tau f_i(x)$.   We have that $h(x)-h(z)=R_\tau(x)-R_\tau(z)+\sum_{i=1}^\tau (f_i(x)-f_i(z)) \le E\tau +L\tau$.  Hence for $x_{\tau+1}\in  \hat{X}^{max}_\tau$ it is sufficient to choose:
\[
\gamma>\gamma_0= \frac{(E+L)\tau-1}{k_{\tau}\eta}\le \frac{E+L+1/\tau}{\beta}\le \frac{E+L+1}{\beta}.
\]   
When Condition \ref{cond:penalty1} holds, $\beta>0$.   
\end{proof}

Theorem \ref{th:constraint} states a lower bound on $\gamma$ in terms of constants $E$, $L$ and $\beta$.   For a quadratic regulariser $R_\tau(x)=\sqrt{\tau}\|x\|_2^2$ we can choose $E=\max_{y,z\in D}(\|y\|_2^2-\|z\|_2^2)$.   Since functions $f_i$ are uniformly Lipschitz then $|f_i(z)-f_i(y)|\le L_f\|z-y\| \le L_f\|D\|$ and so we can choose $L=L_f\|D\|$.  A value for $\beta$ may be unknown but to apply Theorem \ref{th:constraint} in practice we just need to select $\gamma$ large enough, so a pragmatic approach is simply to make $\gamma$ grow with time and then freeze it when it is large enough i.e. when the constraint violations are observed to cease.   

%%%%%%%%%%%%%%%%%%%%%%%%%%%%%%
\subsection{Main Result: Penalised FTRL $O(\sqrt{t})$ Regret \& Violation}\label{sec:main}

Our main result extends the standard FTRL analysis to time-varying constraints:
%%%%%%%%%%%%%%%
\begin{theorem}[Penalised FTRL]\label{th:itworks}
Assume Conditions \ref{cond:ftrl} and \ref{cond:penalty1} hold for $F_i(x)=f_i(x)+\gamma h_i(x)$ with $h_i(x)=\sum_{j=1}^m\max\{0,\frac{1}{i}\sum_{k=1}^i g_k^{(j)}(x)\}$, and the constraints $g_i^{(j)}$ are uniformly Lipschitz. Let the sequence of actions $\{x_\tau\}_{\tau=1}^t$ be generated by the Penalised FTRL update:
\begin{align}
x_{\tau+1} \in \arg\min_{x\in D} R_\tau(x)+\sum_{i=1}^\tau F_i(x)\label{eq:th_ftrl}
\end{align}
Then, if $\gamma$ is sufficiently large, the regret and constraint violation satisfy:
\begin{align*}
&\R_t := \sum_{i=1}^t f_i(x_i) - f_i(y) \le  O(\sqrt{t}), \qquad
\V_t :=  \sum_{i=1}^{t} h_i(x_{i}) \le O(\sqrt{t}), \quad \forall y\in \hat{X}^{max}_t \\
&\hat{X}^{max}_t\!\!=\! \Big\{x\in D: \sum_{k=1}^{i} g_k^{(j)}(x)\!\le\! 0, \forall i\leq t, j\leq m \Big\}\!\!=\!\Big\{x\in D: \sum_{k=1}^i h_k(x)\!=\! 0, \forall i\leq t \Big\}
\end{align*}
 
\end{theorem}
\begin{proof}
\textit{Regret}: Applying Lemma \ref{th:ftrl} then $\sum_{i=1}^t F_i(x_i) - F_i(y) \le O(\sqrt{t})$ for all $y\in D$.   This holds in particular for all $y\in \hat{X}^{max}_t$ and for these points $\sum_{i=1}^t  F_i(y) = \sum_{i=1}^t  f_i(y)$.  Therefore, $\sum_{i=1}^t F_i(x_i) - f_i(y)  \le O(\sqrt{t})$ i.e. $\R_t=\sum_{i=1}^t f_i(x_i) - f_i(y)  \le \ O(\sqrt{t})-\gamma \sum_{i=1}^t h_i(x_i) \le O(\sqrt{t})$ since $h_i(x_i)\ge 0$.     

\textit{Constraint Violation}:  By Theorem \ref{th:constraint}, $x_{\tau+1}\in \hat{X}^{max}_\tau$ for $\tau>t_\epsilon$.   Our interest is in bounding the violation of $\hat{X}^{max}_{\tau+1}$ by $x_{\tau+1}$.  We can ignore the finite interval from 1 to $t_\epsilon$ since it will incur at most a finite constraint violation and so not affect an $O(\sqrt{t})$ bound i.e. when obtaining the $O(\sqrt{t})$ bound we can take $t_\epsilon=0$.   We follow a ``Be-The-Leader'' type of approach and apply Lemma \ref{lem:btl} with ${F}_i(x)=h_i(x)$.  We have that $h_i(x)\ge 0$ and by Condition \ref{cond:penalty1}, there exists a Slater point $z\in D$ such that $h_i(z)=0$, $i=1,\dots,t$.  Hence, $\min_{x\in D} \sum_{i=1}^\tau {F}_i(x)\!=\!0$ and $\arg\min_{x\in D} \sum_{i=1}^\tau {F}_i(x)$ is not empty.  Now,  $x_{\tau+1}\in \hat{X}^{max}_\tau\!=\!\big\{x\in D: \sum_{i=1}^\tau h_i(x)= 0\big\}=\arg\min_{x\in D} \sum_{i=1}^\tau {h}_i(x)$ i.e. $x_{\tau+1}$ is a Follow-The-Leader update with respect to $\sum_{i=1}^\tau {h}_i(x)$.   Hence, by Lemma \ref{lem:btl}, it is $ \sum_{i=1}^{t}h_i(y)  \!\ge \! \sum_{i=1}^{t}h_i(x_{i+1})$, $\forall y\!\in\! D$.  Multiplying both sides of this inequality by -1 and adding $\sum_{i=1}^{t}h_i(x_{i})$, it follows that:
$$
\sum_{i=1}^{t}  \Big(h_i(x_{i}) -  h_i(y) \Big) \le \sum_{i=1}^{t} \Big(h_i(x_{i})- h_i(x_{i+1})\Big) \,\,\,\forall y\in D.
$$
In particular, for $y\in \hat{X}^{max}_t$ then $\sum_{i=1}^{t}h_i(y)=0$ and so
$$
\V_t =\sum_{i=1}^{t} h_i(x_{i}) \le \sum_{i=1}^{t} \Big(h_i(x_{i})-h_i(x_{i+1})\Big).
$$
Since $g_i^{(j)}$ is uniformly Lipschitz then by Lemma \ref{lem:maxlips}, we get that $h_i$ is uniformly Lipschitz, i.e. $|h_i(x_{i})-h_i(x_{i+1})| \le L_g\|x_i-x_{i+1}\|$ and $\V_t \le L_g \sum_{i=1}^{t}\|x_i-x_{i+1}\|$, where $L_g$ is the Lipschitz constant. Since the regularizer $R_\tau(x)$ in the Penalized FTRL update is $\sqrt{\tau}$-strongly convex, by Lemma \ref{lem:reg} we get that $\|x_i-x_{i+1}\|$ is $O(1/\sqrt{i})$ and so $\sum_{i=1}^{t}\|x_i-x_{i+1}\|$ is $O(\sqrt{t})$.  Hence, $\V_t \le O(\sqrt{t})$ as claimed.
\end{proof}

We can immediately generalize Theorem \ref{th:itworks} by observing that a sequence of constraints $\{g_i^{(j)}\}$ which are active at no more than $O(\sqrt{t})$ time steps can be violated while still maintaining $O(\sqrt{t})$ overall sum-violation.  

%%%%%%%%%%%%%%%
\begin{corollary}[Relaxation]\label{th:itworks2}
Define the sets
\[
P_-=\{j:\sum_{i=1}^t\max\{0,\frac{1}{i}\sum_{k=1}^i g_k^{(j)}(x)\}\le O(\sqrt{t})\},\,\,\,\text{and} \,\,\,\, P_+=\{1,\dots,m\}\setminus P_-.
\]
In Theorem \ref{th:itworks} relax Condition \ref{cond:penalty1} so that it only holds for the subset $P_+$ of constraints.   Then the Penalised FTRL update still ensures $O(\sqrt{t})$ regret and constraint violation with respect to:
\[
 \hat{X}^{max}_t=\Big\{x\in D: \sum_{k=1}^{i} g_k^{(j)}(x)\le 0, i=1,\dots,t,j\in P_+\Big\}.
 \]
\end{corollary}

In effect, Corollary \ref{th:itworks2} says that we only need Condition \ref{cond:penalty1} to hold for a \emph{subset} of the constraints (i.e. subset $P_+$).  The effect will be to increase the sum-violation, but only by $O(\sqrt{t})$.  This is the key advantage of the penalty-based approach, namely it allows a soft trade-off between sum-constraint satisfaction/violation, Condition \ref{cond:penalty1} and benchmark set $\hat{X}^{max}_t$. Importantly, note that the Penalised FTRL update itself remains unchanged and does not require knowledge of the partitioning of constraints into sets $P_+$ and $P_-$.

With this in mind,  it is worth noting that we also have the flexibility to partition the constraints in other ways.  For example:
\begin{corollary}\label{th:itworks3}
Consider the setup in Theorem \ref{th:itworks} but using penalty
$$
h_i(x)=\sum_{j=1}^m\max\left\{0,\frac{1}{i}\sum_{k=1}^i g_k^{(j)}(x)\right\}+\delta_i^{(j)}(x)
$$
Then the Penalised FTRL update ensures regret and violation
\begin{align*}
&\R_t := \sum_{i=1}^t \Big( f_i(x_i) - f_i(y)\Big) \le O(\sqrt{t}) - \sum_{i=1}^{t}\sum_{j=1}^m\Big(\delta_i^{(j)}(x_i)-\delta_i^{(j)}(y)\Big)\\
&\V_t :=  \sum_{i=1}^{t} h_i(x_{i}) \le O(\sqrt{t}) +  \sum_{i=1}^{t}\sum_{j=1}^{m}\delta_i^{(j)}(x)
\end{align*}
for all $y\in \hat{X}^{max}_t=\Big \{x\in D: \sum_{k=1}^{i} g_k^{(j)}(x)\le 0, i=1,\dots,t,j=1,\dots,m\Big\}$.
\end{corollary}

When $\delta_i^{(j)}\!\le\! O(1/\sqrt{t})$ then Corollary \ref{th:itworks3} shows that the Penalised FTRL update achieves $O(\sqrt{t})$ regret and violtion, this Corollary will prove useful in the next section.   Other variations of this sort are also possible.

%%%%%%%%%%%%%%%%%%%%%%%%%%%%%%%
\subsection{Necessity of Penalty Growth Condition}
Condition \ref{cond:penalty1} is necessary for Theorems \ref{th:constraint} and \ref{th:itworks} to hold in the sense that when the condition is violated then there exist examples where these theorems fail.

Returning again to the example from the Introduction, selecting $h_i(x)$ according to (\ref{eq:ftrl_pen}) then $h_i(x)=\max\{0,-0.01\} +\max\{0,\frac{n_{2,i}}{i} x\}=\max\{0,\frac{n_{2,i}}{i} x\}$.   Hence, the penalty $\sum_{i=1}^\tau h_i(x)\le \sum_{i=1}^\tau\frac{n_{2,i}}{i}x$.  When $n_{2,i}<O(i)$ then $\sum_{i=1}^\tau h_i(x)<O(\tau)$ (since $\sum_{i=1}^\tau \frac{1}{i^{c}} \le \int_0^\tau \frac{1}{i^{c}} di = \frac{\tau^{1-c}}{1-c}$ for $0\le c\le1$) and Condition \ref{cond:penalty1} is violated (since $k_\tau\le n_{2,\tau}<O(\tau)$ and so there does not exist any $\beta>0$ such that $k_\tau \ge \frac{\beta}{\eta}\tau$).  For $\tau$ large enough the penalty $\sum_{i=1}^\tau h_i(x)$ therefore inevitably becomes small relative to $\sum_{i=1}^\tau f_i(x) = -2\tau x$, which leads to persistent violation of constraint $x\le 0$ i.e. Theorem \ref{th:constraint} fails.   This is what we see in Figure \ref{fig:ex}(a).    

When $n_{2,i}\le O(\sqrt{i})$ then $\frac{n_{2,i}}{i}\le O(1/\sqrt{i})$ and the constraint sum-violation $\sum_{i=1}^\tau h_i(x)\le O(\sqrt{i})$.  Hence, Corollary \ref{th:itworks2} still works even though Theorem \ref{th:constraint} fails.  However, when $n_{2,i}$ greater than $O(\sqrt{i})$ but less than $O(i)$ then the constraint violation is greater than $O(\sqrt{i})$ and so Corollary \ref{th:itworks2} also fails.

It is worth noting that while we might consider gaining penalty growth by scaling $\gamma$ with $t$ this in inadmissible because Condition \ref{cond:ftrl} requires $F_t(x)=f_t(x)+\gamma h_t(x)$ to be uniformly Lipschitz i.e. for the same Lipschitz constant to apply at all times $t$.  

%%%%%%%%%%%%%%%%%%%
\subsection{Constraints Satisfying Penalty Growth Condition} 
A natural question to ask is which classes of time-varying constraints satisfy Condition \ref{cond:penalty1}.  In this section we present some useful examples.  In particular, we consider the classes of constraints considered by~\cite{pmlr-v108-valls20a} and~\cite{yu2017}, since these are the only previous works for time-varying constraints that report $\R_t, \V_t=O(\sqrt{t})$.

%%%%%%%%%%%%%%%%%%%
\subsubsection{Perturbed Constraints}
In~\cite{pmlr-v108-valls20a} the considered constraints are of the form:
$$
g_i^{(j)}(x) = g^{(j)}(x) + b_i^{(j)}
$$
with common Slater point and $b_i^{(j)}$ upper bounded by some value, i.e., $b_i^{(j)} \! \le \! \bar{b}^{(j)}, \forall i$.  For this class of constraints we have that:
\begin{align*}
h_i(x)&=\sum_{j=1}^m\max\left\{0,\frac{1}{i}\sum_{k=1}^i (g^{(j)}(x) + b_k^{(j)})\right\} 
=  \sum_{j=1}^m\max\left\{0,g^{(j)}(x)+\frac{1}{i}\sum_{k=1}^i b_k^{(j)}\right\}
\end{align*}
Defining $\underline{b}^{(j)}_t=\frac{1}{t}\sum_{k=1}^t b_k^{(j)}$ and $\Delta_i^{(j)}(x)=\frac{1}{i}\sum_{k=1}^i (b_k^{(j)}-\underline{b}_t)$, then we can rewrite the penalty equivalently as
$$
h_i(x)=\sum_{j=1}^m\max\left\{0,g^{(j)}(x)+\underline{b}_t^{(j)}\right\}+\delta_i^{(j)}(x)
$$
with
$\delta_i^{(j)}(x)=\max\big\{0,g^{(j)}(x)+\underline{b}_t^{(j)} + \Delta_i^{(j)}(x)\big\}-\max\big\{0,g^{(j)}(x)+\underline{b}_t^{(j)}\big \}$.  When $|\Delta_i^{(j)}(x)|$ is $O(1/\sqrt{i})$ then, by Lemma \ref{lem:maxlips}, so is $|\delta_i^{(j)}(x)|$.   Hence, when $|\Delta_i^{(j)}(x)|$ is $O(1/\sqrt{i})$ then we can use the fact that Condition \ref{cond:penalty1} holds for constraints $g^{(j)}(x)+\underline{b}^{(j)}_t\le 0$ to show, by Corollary \ref{th:itworks3}, that the Penalised FTRL update achieves $O(\sqrt{t})$ regret and violation with respect to benchmark set $\hat{X}_t^{max}=\{x:g^{(j)}(x)+\underline{b}^{(j)}_t\le 0\}$.  This corresponds to one extreme of~\cite{pmlr-v108-valls20a}'s benchmark but Theorem \ref{th:itworks2} provides more general conditions under which it is applicable, while \cite{pmlr-v108-valls20a} only considers constraints that are either time-invariant or i.i.d.

Alternatively, defining $\Delta_i^{(j)}(x)=\frac{1}{i}\sum_{k=1}^i (b_k^{(j)}-\bar{b}^{(j)})$ and we can rewrite the penalty equivalently as
$$
h_i(x)=\sum_{j=1}^m\max\left\{0,g^{(j)}(x)+\bar{b}^{(j)}\right\}+\delta_i^{(j)}(x)
$$
with 
$\delta_i^{(j)}(x)=\max\{0,g^{(j)}(x)+\bar{b}^{(j)} + \Delta_i^{(j)}(x)\}-\max\{0,g^{(j)}(x)+\bar{b}^{(j)} \}$.  Observe that $\delta_i^{(j)}(x)\le0$ since $\Delta_i^{(j)}(x)\le 0$.  Hence, $\delta_i^{(j)}(x)$ does not add to the upper bound on the sum-constraint violation and so, by Corollary \ref{th:itworks3}, that the Penalised FTRL update achieves $O(\sqrt{t})$ regret and violation with respect to benchmark set $\hat{X}_t^{max}=\{x:g^{(j)}(x)+\bar{b}^{(j)}\le 0\}$.  This corresponds to the other extreme of~\cite{pmlr-v108-valls20a}'s benchmark, and in fact corresponds to the weak benchmark $X_t^{min}$ and so is perhaps less interesting.

%%%%%%%%%%%%%%%%
\subsubsection{Families Of Constraints } 

Suppose the time-varying constraint functions $g_{i}^{(j)}$ are selected from some family.  That is, let $A^{(j)}=\{a_{1}^{(j)},\dots,a_{n_j}^{(j)}\}$ be a family of functions indexed by $k=1,\dots,n_j$ with $a_{k}^{(j)}:D\rightarrow\mathbb{R}$ being $L_g$-Lipschitz and $|a_{k}^{(j)}(x)|\le a_{max}$ for all $x\in D$.  At time $i$, constraint $g_{i}^{(j)}=a_{k}^{(j)}$ for some $k\in\{1,\dots,n_j\}$, i.e. at each time step the constraint $g_{i}^{(j)}$ is selected from family $A^{(j)}$.  Let $n_{k,\tau}^{(j)}$ denote the number of times that function $a_{k}^{(j)}$is visited up to time $\tau$ and $p_{k,\tau}^{(j)}=n_{k,\tau}^{(j)}/\tau$ the fraction of times that $a_{k}^{(j)}$ is visited.  With this setup the penalty is:
\begin{align*}
h_i(x)&=\sum_{j=1}^m\max\left\{0,\frac{1}{i}\sum_{k=1}^i g_{i}^{(j)}\right\}
=\sum_{j=1}^m\max\left\{0,\sum_{k=1}^{n_j} p_{k,i}^{(j)} a_{k}^{(j)}(x) \right\}
\end{align*}
We proceed by rewriting the penalty equivalently as
$$
h_i(x)=\sum_{j=1}^m\max\left\{0,\sum_{k=1}^{n_j} p_{k}^{(j)} a_{k}^{(j)}(x) \right \}+\delta_i^{(j)}(x)
$$
with
$
\delta_i^{(j)}(x) = \max\big\{0,\sum_{k=1}^{n_j} p_{k,i}^{(j)} a_{k}^{(j)}(x)\big \}
-\max\big\{0,\sum_{k=1}^{n_j} p_{k}^{(j)} a_{k}^{(j)}(x)\big \}
$
By Lemma \ref{lem:maxlips}, $|\delta_i^{(j)}(x)|\le |\sum_{k=1}^{n_j} (p_{k,i}^{(j)}-p_{k}^{(j)} ) a_{k}^{(j)}(x)|$.   Assume the following condition holds:

\begin{condition}[$1/\sqrt{t}$-Convergence]\label{cond:sqrt}
For $\epsilon>0$ there exists $t_0>0$ and $0\le p_{k}^{(j)}\le 1$, $\sum_{j=1}^m\sum_{k=1}^{n_j}p_{k}^{(j)}=1$ such that $|p_{k,\tau}^{(j)}- p_{k}^{(j)}|\le\epsilon/\sqrt{\tau}$ for all $\tau>t_0$.
\end{condition}
Then for all $\tau>t_0$,
$
|\delta_i^{(j)}(x)|\le n_j \frac{\epsilon}{\sqrt{\tau}}a_{max}\le \bar{n} \frac{\epsilon}{\sqrt{\tau}}a_{max}
$
with $\bar{n} :=\max_j n_j$.  By Corollary \ref{th:itworks3} it now follows that Penalised FTRL achieves $O(\sqrt{t})$ regret and violation with respect to benchmark $\hat{X}^{max}_t=\big\{x:\sum_{k=1}^{n_j} p_{k}^{(j)} a_{k}^{(j)}(x) \le 0
\big\}$.  Observe that in this case $\hat{X}^{max}_t={X}^{max}_\infty$, i.e., we obtain $O(\sqrt{t})$ regret and violation with respect to the strong benchmark, which is very appealing.    Note that we don't need to know the relative frequencies in advance for this analysis to work.

%%%%%
\subsubsection{Example}
Suppose $D\!=\![-10,10]$, loss function $f_\tau(x)\!=\!-2x$ and constraint $g_\tau(x)$ alternates between $a_1(x)\!=\!-0.01$ and $a_2(x)\!=\!x$, equaling $a_2(x)$ at time $\tau$ with probability\footnote{Recall that $c\sum_{\tau=0}^t \frac{1}{\tau^{1-c}} \approx c\int_0^t \frac{1}{\tau^{1-c}} d\tau = t^{c}$ for $0\le c\le1$.   Hence, with this choice $E[n_{2,t}]\approx 0.1 t^{c}$ and $E[p_{2,t}]\approx 0.1t^{c-1}$.\label{foot:bound}} $0.1c/\tau^{1-c}$.   Figure \ref{fig:ex}(a) shows the performance vs $c$ of the Penalised FTRL update with quadratic regulariser $R_\tau(x)=\sqrt{\tau} x^2$ and $F_\tau(x)=f_\tau(x)+\gamma\max\{0,p_{1,\tau}a_1(x) + p_{2,\tau}a_2(x)\}$ with parameter $\gamma=25$.   It can be seen that for $c=1$ and $c=0.5$ the constraint violation is well-behaved, staying close to zero, but for $c\!\in\!(0.5,1)$ the constraint violation grows with time.

What is happening here is that when $c\!=\!1$ then $p_{1,\tau}\rightarrow0.9$, $p_{2,\tau}\rightarrow0.1$ and the penalty term $\gamma\max\{0,p_{1,\tau}a_1(x) + p_{2,\tau}a_2(x)\}$ in  $F_\tau(x)$ ensures the violation $\sum_{i=1}^t g_i(x) =t(p_{1,t}a_1(x) + p_{2,t}a_2(x))$ stays small.  When $c=0.5$, then $p_{1,\tau}\rightarrow1$, $p_{2,\tau}\rightarrow0$ and the penalty term ensures $tp_{1,t}a_1(x)$ stays small while $t p_{2,t}a_2(x)$ is $O(\sqrt{t})$, thus $\sum_{i=1}^t g_i(x)$ is $O(\sqrt{t})$.   When $c\!\in (0.5,1)$ then again $p_{1,\tau}\rightarrow1$, $p_{2,\tau}\rightarrow0$ and the penalty term ensures $tp_{1,t}a_1(x)$ stays small but now $tp_{2,t}a_2(x)$ is larger than $O(\sqrt{t})$ and so $\sum_{i=1}^t g_i(x)$ is also larger than $O(\sqrt{t})$.  

We claim that $1/\sqrt{t}$-convergence is sufficient for Penalised FTRL to achieve $O(\sqrt{t})$ regret and violation with respect to $X^*$, but it remains an open question whether or not it is also a necessary condition.  Nevertheless, in simulations we observe that when $1/\sqrt{t}$-convergence does not hold then performance is often poor and that this is not specific to the FTRL algorithm, e.g. Figure \ref{fig:ex}(b) illustrates the performance of the canonical online primal-dual update (e.g. see~\cite{pmlr-v108-valls20a}),
\begin{align}
x_{t+1}&=\Pi_D\left(x_t - \alpha_t (\partial f_t(x_t) + \lambda_t\partial g_t(x_t))\right),\  %\label{eq:pm1}\\
\lambda_{t+1} = \Big[\lambda_t + \alpha_tg_t(x_{t+1}\Big]^+\label{eq:pm1}
\end{align}
where $\Pi_D$ denotes projection onto set $D$ and step size $\alpha_t=5/\sqrt{t}$.  

\begin{figure}
\centering
\subfigure[Penalised FTRL]{
\includegraphics[width=0.45\columnwidth]{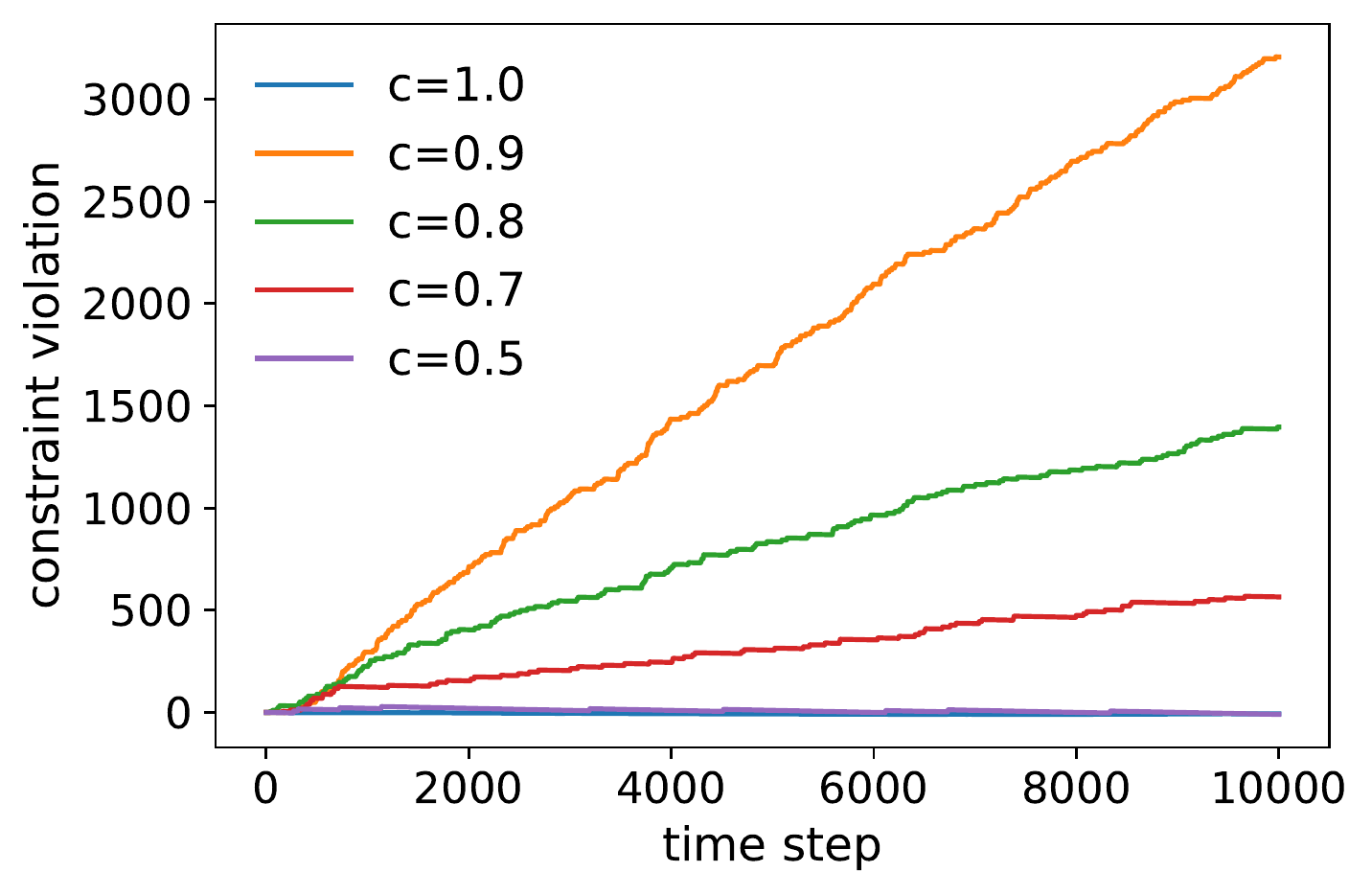}
}
\subfigure[Primal-Dual]{
\includegraphics[width=0.45\columnwidth]{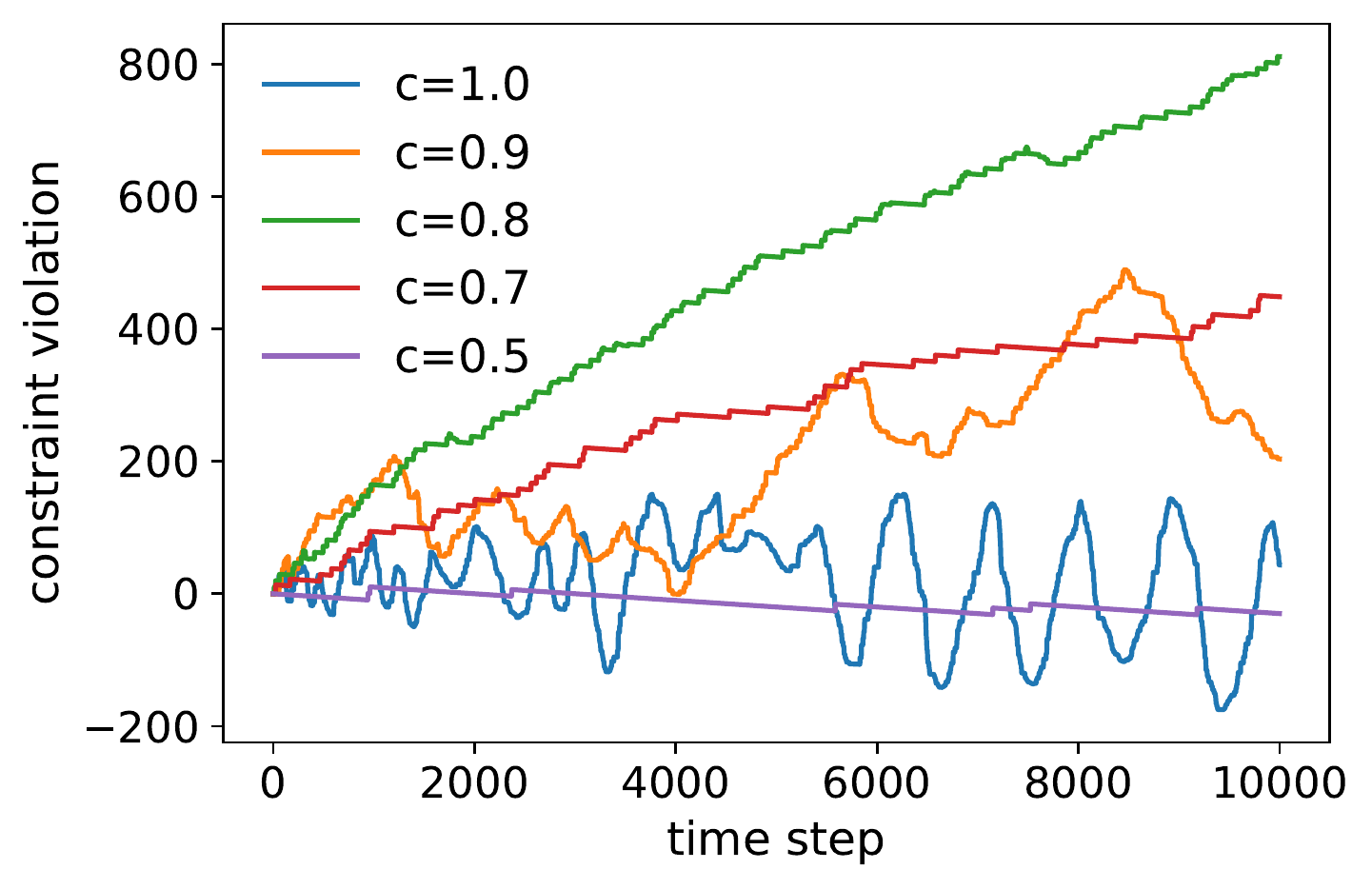}
}
\vspace{-2mm}
\caption{Example about the role of $1/\sqrt{t}$-convergence in achieving $\V_t=O(\sqrt{t})$.}\label{fig:ex}
\vspace{-2mm}
\end{figure}

%%%%%%%%%%%%%%%%%%%
\subsubsection{I.i.d Stochastic Constraints}
In~\cite{yu2017} i.i.d. constraint functions drawn from a family are considered and a primal-dual algorithm is presented that achieves $O(\sqrt{t})$ regret and expected violation.  Since with high probability the empirical mean converges at rate $1/\sqrt{t}$ with high probability we can immediately apply the foregoing analysis to the sample paths to show that Penalised FTRL achieves $O(\sqrt{t})$ regret and violation with respect to ${X}^{max}_t$ with high probability.  In more detail, let indicator random variiable $I^{j}_{k,i}=1$ when constraint function $a_k^{(j)}$ is selected at time $i$, and otherwise  $I^{(j)}_{k,i}=0$. By the law of large numbers (we can use any convenient concentration inequality, e.g. Chebyshev), with high probabilty the emprical mean satisfies $|\frac{1}{\tau}\sum_{i=1}^\tau I^{(j)}_{k,i} - p^{j}_{k}]|\le 1/\sqrt{\tau}$ with high probability.  That is, Condition \ref{cond:sqrt} holds with high probability and we are done.

%%%%%%%%%%%%%%%%%%%
\subsubsection{Periodic Constraints}
Let indicator $I^{j}_{k,i}=1$ when constraint function $a_k^{(j)}$ is selected at time $i$, and otherwise  $I^{(j)}_{k,i}=0$.  When the constraints are visitied in a periodic fashion then 
$$I^{(j)}_{k,i}=\begin{cases}
1 & i=nT_k^{(j)}, n=1,2,\dots\\
0 & \text{otherwise}
\end{cases}
$$
where $T_k^{(j)}$ is the period of constraint $a_k^{(j)}$.   Then 
$
|\frac{1}{\tau}\sum_{i=1}^\tau I^{(j)}_{k,i} -\frac{1}{T_k^{(j)}}|
%=| \frac{1}{\tau}\lfloor \frac{\tau}{T_k^{(j)}} \rfloor -\frac{1}{T_k^{(j)}}|
=\frac{1}{\tau}| \lfloor \frac{\tau}{T_k^{(j)}} \rfloor -\frac{\tau}{T_k^{(j)}}|
\le  \frac{1}{\tau } 
$. 
Hence Condition \ref{cond:sqrt} holds and we are done.

%%%%%%%%%%%%%%%%%%%
%%%%%%%%%%%%%%%%%
\section{Summary and Conclusions}
In this paper we extend the classical FTRL algorithm to encompass time-varying constraints by leveraging, for the first time in this context, the seminal penalty method of \cite{zangwill67}.   We establish sufficient conditions for this new Penalised FTRL algorithm to achieve $O(\sqrt{t})$ regret and violation with respect to a strong benchmark $\hat{X}^{max}_t$ that expands significantly the previously-employed benchmarks in the literature.  This result matches the performance of the best existing primal-dual algorithms in terms of regret and constraint violation growth rates , while substantially extending the class of problems covered.    The key to this improvement lies in how the time-varying constraints are incorporated into the FTRL algorithm.  We conjecture that adopting a similar formulation with a primal-dual algorithm, namely using:
\begin{align*}
x_{t+1}&=\Pi_D(x_t - \alpha_t(\partial f_t(x_t) + \lambda_t \partial h_t(x_t))),\ 
\lambda_{t+1} = [\lambda_t + \alpha_th_t(x_{t+1}]^+
\end{align*}
where $h_t(x)=\frac{1}{t}\sum_{i=1}^tg_i(x_{t+1}$, would allow similar performance to be achieved by primal-dual algorithms as by FTRL but we leave this to future work.

\bibliographystyle{splncs04}

\bibliography{references}

\begin{thebibliography}{10}
\providecommand{\url}[1]{\texttt{#1}}
\providecommand{\urlprefix}{URL }
\providecommand{\doi}[1]{https://doi.org/#1}

\bibitem{chen2017}
Chen, T., Ling, Q., Giannakis, G.B.: An online convex optimization approach to
  proactive network resource allocation. IEEE Transactions on Signal Processing
   \textbf{65}(24),  6350--6364 (2017)

\bibitem{hazan-book}
Hazan, E.: Introduction to online convex optimization. Foundations and Trends
  in Optimization  \textbf{2},  157--325 (2016)

\bibitem{pmlr-v48-jenatton16}
Jenatton, R., Huang, J.C., Archambeau, C.: Adaptive algorithms for online
  convex optimization with long- term constraints. In: Proc. of ICML. pp.
  402--411 (2016)

\bibitem{pmlr-v97-liakopoulos19a}
Liakopoulos, N., Destounis, A., Paschos, G., Spyropoulos, T., Mertikopoulos,
  P.: Cautious regret minimization: Online optimization with long-term budget
  constraints. In: Proceedings of ICML. pp. 3944--3952 (2019)

\bibitem{JMLR:v13:mahdavi12a}
Mahdavi, M., Jin, R., Yang, T.: Trading regret for efficiency: Online convex
  optimization with long term constraints. Journal of Machine Learning Research
   \textbf{13}(81),  2503--2528 (2012)

\bibitem{JMLR:v10:mannor09a}
Mannor, S., Tsitsiklis, J.N., Yu, J.Y.: Online learning with sample path
  constraints. Journal of Machine Learning Research  \textbf{10}(20),  569--590
  (2009)

\bibitem{mcmahan-survey17}
McMahan, H.B.: A survey of algorithms and analysis for adaptive online
  learning. Journal of Machine Learning Research  \textbf{18},  1--50 (2017)

\bibitem{shwartz-book}
Shalev-Shwartz, S.: Online learning and online convex optimization. Foundations
  and Trends in Optimization  \textbf{4},  107--194 (2011)

\bibitem{pmlr-v70-sun17a}
Sun, W., Dey, D., Kapoor, A.: Safety-aware algorithms for adversarial
  contextual bandit. In: Proc. of ICML. pp. 3280--3288 (2017)

\bibitem{pmlr-v108-valls20a}
Valls, V., Iosifidis, G., Leith, D., Tassiulas, L.: Online convex optimization
  with perturbed constraints: Optimal rates against stronger benchmarks. In:
  Proceedings of AISTATS. pp. 2885--2895 (2020)

\bibitem{zangwill67}
W.J.Zangwill: {Nonlinear Programming via Penalty Functions}. Management Science
   \textbf{13}(5),  344--358 (1967)

\bibitem{Johansson2020}
Yi, X., Li, X., Xie, L., Johansson, K.H.: Distributed online convex
  optimization with time-varying coupled inequality constraints. IEEE
  Transactions on Signal Processing  \textbf{68},  731--746 (2020)

\bibitem{yu2017}
Yu, H., Nelly, M., Wei, X.: Online convex optimization with stochastic
  constraints. In: Proceedings of NIPS (2017)

\bibitem{zinkevich}
Zinkevich, M.: Online convex programming and generalized infinitesimal gradient
  ascent. In: Proc. of ICML (2003)

\end{thebibliography}

%%%%%%%%%%%%%%%%%
\section*{Appendix A: Proofs}

%%%%%%%%%%%%%%%%%
\subsection{Proof of Lemma \ref{lem:zangwill}}
\begin{proof}
Firstly note that for feasible points $x\in X$ we have that $g^{(j)}(x)\le0$, $j=1,\cdots,m$ and so $F(x)=f(x)$.  By definition $f(x) \ge f^*=\inf_{x\in X} f(x)$ and so the stated result holds trivially for such points.   Now consider an infeasible point $w\notin X$.   Let $z$ be an interior point satisfying $g^{(j)}(z)<0$, $j=1,\cdots,m$; by assumption such a point exists.   Let ${\gamma}_0 = \frac{f^*-f(z)-1}{G}$.   It is sufficient to show that $F(w)> f^*$ for ${\gamma}\ge {\gamma}_0$ and $G=\max_{j\in\{1,\cdots,m\}}\{g^{(j)}(z)\}$.    

Let $v=\beta z + (1-\beta) w$ be a point on the chord between points $w$ and $z$, with $\beta\in(0,1)$ and $v$ on the boundary of $X$ (that is $g^{(j)}(v)\le 0$ for all $j=1,\cdots,m$ and $g^{(j)}(v)=0$ for at least one $j\in\{1,\cdots,m\}$).  Such a point $v$ exists since $z$ lies in the interior of $X$ and $w\notin X$.    Let $A:=\{j:j\in\{1,\cdots,m\},g^{(j)}(v)=0\}$ and $t(x):=f(x)+{\gamma}\sum_{j\in A} g^{(j)}(x)$.   Then $t(v)=f(v)\ge f^*$.   Also, by the convexity of $g^{(j)}(\cdot)$ we have that for $j\in A$ that $g^{(j)}(v) = 0 \le \beta g^{(j)}(z) + (1-\beta) g^{(j)}(w)$.  Since $g^{(j)}(z)<0$, it follows that $g^{(j)}(w)>0$.  Hence, $\sum_{j\in A}g^{(j)}(w) = \sum_{j\in A}\max\{0,g^{(j)}(w)\} \le  \sum_{j=1}^m\max\{0,g^{(j)}(w)\}$ and so $t(w) \le F(w,{\gamma})$. Now, observe that $t(z)= f(z)+{\gamma}\sum_{j\in A} g^{(j)}(z) \le f(z)+{\gamma}_0\sum_{j\in A} g^{(j)}(z)$ since $g^{(j)}(z)<0$ and ${\gamma}\ge {\gamma}_0$.  Hence,
\begin{align}
t(z) & \le f(z)+(f^*-f(z)-1)\frac{\sum_{j\in A} g^{(j)}(z)}{G} \label{eq:zangwillG}
\end{align}
Selecting $G$ such that $\frac{\sum_{j\in A} g^{(j)}(z)}{G}\ge 1$ then
$
t(z) \le f^*-1 \le t(v) -1
$.
So we have established that $f^*\le t(v)$, $t(z)\le t(v)-1$ and $t(w) \le F(w)$.  Finally, by the convexity of $t(\cdot)$, $t(v) \le \beta t(z) + (1-\beta) t(w)$.    Since $t(z)\le t(v)-1$ it follows that $t(v) \le \beta (t(v)-1) + (1-\beta) t(w)$ i.e. $t(v) \le -\frac{\beta}{1-\beta}+t(w)$.  Therefore $f^* \le -\frac{\beta}{1-\beta} + F(w)<F(w)$ as claimed.
\end{proof}

\end{document}